\documentclass[letterpaper, 10 pt, conference]{ieeeconf}  

\IEEEoverridecommandlockouts                              

\usepackage{graphics} 
\usepackage{epsfig} 
\usepackage{amsmath} 
\usepackage{amssymb}  
\usepackage{tikz}
\usepackage{caption}
\usepackage{subcaption}
\usepackage{url}
\usetikzlibrary{calc, arrows.meta}

\newcommand{\dfb}{\stackrel{\Delta}{=}}
\newcommand{\R}{\ensuremath{\mathbb R}}

\newtheorem{theorem}{\textbf{Theorem}}[section]
\newtheorem{remark}[theorem]{Remark}
\newtheorem{lemma}[theorem]{\textbf{Lemma}}
\newtheorem{assumption}[theorem]{\textbf{Assumption}}

\title{\LARGE \bf
Circular formation control of fixed-wing UAVs with constant speeds.
}

\author{Hector Garcia de Marina$^{1}$ Zhiyong Sun$^{2}$ Murat Bronz$^{1}$ and Gautier Hattenberger$^{1}$
\thanks{$^{1}$H. Garcia de Marina, M. Bronz and G. Hattenberger are with the Ecole National de l'Aviation Civile, University of Toulouse, Toulouse, France. {\tt\small hgdemarina@ieee.org}.}
\thanks{$^{2}$Z. Sun is with the Australian National University, Canberra.  {\tt\small zhiyong.sun@anu.edu.au}.}%
}

\begin{document}

\maketitle
\thispagestyle{empty}
\pagestyle{empty}

\begin{abstract}
In this paper we propose an algorithm for stabilizing circular formations of fixed-wing UAVs with constant speeds. The algorithm is based on the idea of tracking circles with different radii in order to control the inter-vehicle phases with respect to a target circumference. We prove that the desired equilibrium is exponentially stable and thanks to the guidance vector field that guides the vehicles, the algorithm can be extended to other closed trajectories. One of the main advantages of this approach is that the algorithm guarantees the confinement of the team in a specific area, even when communications or sensing among vehicles are lost. We show the effectiveness of the algorithm with an actual formation flight of three aircraft. The algorithm is ready to use for the general public in the open-source Paparazzi autopilot.
\end{abstract}

\section{INTRODUCTION}
There is a growing body of literature that recognises the importance of Unmanned Aerial Vehicles (UAVs) for aerial surveillance, search and rescue, patrol, inspection of facilities, precision agriculture, and atmosphere study \cite{beard2006decentralized,goodrich2008supporting,tokekar2016sensor
}.
The current trend in multi-agent systems \cite{olfati2007consensus} has led to a proliferation of works where individual and independent UAVs have gradually been replaced by a team of cooperative and coordinated ones \cite{ryan2004overview,bertuccelli2004robust,wang2007cooperative,stipanovic2004decentralized}. Formation control is one of the most well-known tools for assessing the cooperation and coordination of multi-agent systems \cite{oh2015survey,anderson2008uav}, where the usage of such a tool aims at forming and maintaining a prescribed geometrical shape for a group of vehicles.

An important challenge in the formation control of UAVs is to consider demanding constraints for the vehicles. For example, fixed-wing UAVs are preferred to rotorcrafts in missions where the endurance or the capacity for traveling long distances are essential requirements. However, rather than point mass models \cite{montenbruck2016fekete,lee2016distributed}, guidance algorithms for fixed-wing UAVs have to consider nonholonomic constraints, e.g., by modeling the aircraft as unicycles. Theoretical contributions on the coordination and formation control of unicycles include the consensus and rendezvous \cite{lin2005necessary,dimarogonas2007rendezvous}, and the circular formations \cite{marshall2006pursuit,el2013distributed}. The circular or pursuit formation design is a practical method for steering the vehicles to coordinated and periodic trajectories, which provide means to sample data with a desired spatial and temporal separation \cite{leonard2010coordinated}. However, all the above mentioned works consider that both the speed and orientation of the vehicles can be actuated by the formation control algorithm in order to accomplish the mission's goal.

Fixed-wing aircraft are usually designed for flying most efficiently at a given fixed airspeed \cite{anderson2010fundamentals}. Furthermore, an aircraft needs to fly with an airspeed over a certain lower bound or otherwise the aircraft stalls and falls down. Consequently, the control problem to be discussed in this paper considers unicycle-like vehicles with constant speed, i.e., we only actuate on the steering of the vehicle via coordinated turns by actuating on the bank angle of the aircraft. Note that having a constant airspeed does not imply to have a constant ground-speed because of the effect of the wind. Therefore, the wind causes the aircraft to travel with different ground speeds depending on its yaw and heading angles with respect to some frame of coordinates fixed on the ground, e.g., at center of the circular formation. Nevertheless, in this paper it will be assumed that the speed of the wind is much smaller than the desired airspeed, so the ground-speed can be considered almost constant during the vehicle's mission.

Circular formations for unicycles with the constraint of having constant speeds 
make the formation control problem more challenging as it has been shown in the early work in \cite{justh2004equilibria} for just two vehicles. In particular, the analysis of constantly moving vehicles is problematic if one wants to control geometrical relations between them. However, a clever strategy consisting in controlling the instantaneous center of rotation of the vehicle, instead of its position, for a given fixed angular velocity has been employed in different works \cite{sepulchre2007stabilization,sepulchre2008stabilization,seyboth2014collective,brinon2014cooperative,sun2015collective}. The benefit of this approach is that such instantaneous center of rotation for a \emph{unit speed} vehicle can be fixed with respect to some global frame of coordinates while the vehicle is still moving, i.e., it just circulates around a constant point. The cited works applying this approach include the rendezvous of such points or the control of geometric relations between them such as position- or distance-based formation control approaches \cite{oh2015survey}. However, there are certain drawbacks associated with the control of such centers of rotation. For example, it is common that during the mission the steering control action of a vehicle could be close to or exactly zero. For such a situation, if communications or lines of sight for the sensors are lost, then the control action of the vehicle is not updated since the steering  depends only on the relative states with respect to its neighbors, and consequently, the aircraft will continue flying straight. This is clearly a problem since nowadays the restrictions for drones in the airspace are tight in many countries. Therefore, the flight plan has to guarantee that the fixed-wing UAVs will not abandon their designated flying area.

In this paper we propose a distributed algorithm for controlling circular formations of fixed-wing UAVs with constant speeds, where each vehicle has a (feasible) prescribed inter-position with respect to its neighbors on the circle. In particular, the formation control algorithm does not directly actuate on steering the vehicle but setting the radius of the circle to be tracked by the vehicle, i.e., we actuate on the angular velocity of the vehicle around the center point of the circle. This approach can be related to algorithms where the agents are exclusively confined on the target circumference and they can change their phases \cite{wang2013forming,bullo2009distributed}. While several works controlling the inter-vehicle phases have considered constraints such as \emph{anonymity}, restrictions in the communications or order preservation, to the best of the authors' knowledge, only the work by Wang \emph{et al} \cite{wang2014controlling} addresses the constraint of having vehicles that cannot move \emph{backwards}. However, such condition implies that the vehicles can be stopped on the circumference, something impossible for a fixed-wing aircraft.

The proposed algorithm in this work has a series of advantages and features. First, it is distributed, i.e., the vehicles depend only on relative measurements, such as relative positions, with respect to their neighbors, and a complete graph is not necessary. Second, unlike many of the above cited works, the desired inter-vehicle angles on the circle can be prescribed. Third, the desired formation is exponentially stable. Fourth, an arbitrary bound of the maximum distance of the vehicles with respect to the center of the circle can be chosen by design, therefore it is guaranteed the confinement of the formation regardless of broken communications or sensing. Fifth, it is possible to extend the algorithm to not only circles, but at least to any smooth closed orbit homeomorphic to a circle.

The remaining parts of the paper are organized as follows. In Section \ref{sec: bac}, the notation, the considered model of the fixed-wing aircraft and the employed trajectory tracking algorithm for following a circle are introduced. In Section \ref{sec: cf} we state the circular formation problem and propose the design of a controller as a solution together with a stability analysis. Experimental results with actual aircraft are presented in Section \ref{sec: exp}. The algorithm has been implemented in the popular open-source autopilot Paparazzi \cite{hattenberger2014using} and it is ready to be used by the general public.

\section{Preliminaries}
\label{sec: bac}

\subsection{Notation}
Consider a formation of $n\geq 2$ fixed-wing UAVs whose positions are denoted by $p_i\in\R^2$ with $i\in\{1,\dots,n\}$. The vehicles are able to sense the relative positions with respect to their neighbors. The neighbors' relationships are described by an undirected graph $\mathbb{G} = (\mathcal{V}, \mathcal{E})$ with the vertex set $\mathcal{V} = \{1, \dots, n\}$ and the ordered edge set $\mathcal{E}\subseteq\mathcal{V}\times\mathcal{V}$. The set $\mathcal{N}_i$ of the neighbors of vehicle $i$ is defined by $\mathcal{N}_i\dfb\{j\in\mathcal{V}:(i,j)\in\mathcal{E}\}$. Two vertices are \emph{adjacent} if $(i,j)\in\mathcal{E}$. A \emph{path} from a vertex $i$ to a vertex $j$ is a sequence starting at $i$ and ending at $j$ such that two consecutive vertices are adjacent, and if $i=j$ then the path is called a \emph{cycle}. We assume that the graph $\mathcal{G}$ is \emph{connected}, i.e., there is a path between any two vertices $i$ and $j$. We define the elements of the incidence matrix $B\in\R^{|\mathcal{V}|\times|\mathcal{E}|}$, where $|\mathcal{X}|$ denotes the cardinality of the set $\mathcal{X}$, for $\mathbb{G}$ by
\begin{equation}
	b_{ik} \dfb \begin{cases}+1 \quad \text{if} \quad i = {\mathcal{E}_k^{\text{tail}}} \\
		-1 \quad \text{if} \quad i = {\mathcal{E}_k^{\text{head}}} \\
		0 \quad \text{otherwise}
	\end{cases},
	\label{eq: B}
\end{equation}
where $\mathcal{E}_k^{\text{tail}}$ and $\mathcal{E}_k^{\text{head}}$ denote the tail and head nodes, respectively, of the edge $\mathcal{E}_k$, i.e., $\mathcal{E}_k = (\mathcal{E}_k^{\text{tail}},\mathcal{E}_k^{\text{head}})$.

A circular trajectory with radius $r\in\R^+$ can be described by the following implicit equation
\begin{equation}
	\mathcal{C}_r \dfb \{p \,:\, \varphi(p) =  0\},
	\label{eq: cr}
\end{equation}
where $\varphi(p) = p_x^2 + p_y^2 - r^2$ and $p = \begin{bmatrix}p_x & p_y \end{bmatrix}^T$ is a Cartesian position with respect to a frame of coordinates whose origin  is at the center of $\mathcal{C}_r$. The plane $\R^2$ can be covered by the following disjoint sets $\varphi_c(p) \dfb \varphi(p) = c \in\R$, where each \emph{level set} is defined for a particular value of $c$ such that the resulting circle's radius is non-negative, and in particular, the \emph{zero level set} corresponds uniquely to $\mathcal{C}_r$. We define by $n(p)\dfb\nabla\varphi(p)$ the normal vector to the circle corresponding to the level set $\varphi(p)$ and the tangent vector at the same point $p$ is given by the rotation
\begin{equation}
	\tau(p) = En(p), \quad E=\begin{bmatrix}0 & 1 \\ -1 & 0\end{bmatrix}. \nonumber
\end{equation}
Note that $\mathcal{C}_r$ belongs to the $C^2$ space and it is regular everywhere excepting at its center, i.e.,
\begin{equation}
	\nabla\varphi(p) \neq 0 \iff p\neq 0,
\end{equation}
and all the level sets $\varphi_c(p)$ can be parametrized. In particular, the vehicle $i$ can calculate such 
parametrization associated to its position with the following expression
\begin{equation}
	\theta_i(p) = \operatorname{atan2}(p_y,p_x)\in(-\pi,\pi].
	\label{eq: thetai}
\end{equation}
Note that $\theta_i(p)$ belongs to the circle group $\mathbb{S}^1$.

\subsection{Fixed-wing aircraft's model}
Consider for the \emph{unit speed} $i$'th fixed-wing aircraft the following nonholonomic model in 2D
\begin{equation}
\begin{cases}
	\dot p_i &= m(\psi_i) \\
	\dot\psi_i &= u_{\psi_i},
\end{cases}
	\label{eq: pdyn}
\end{equation}
where $m = \begin{bmatrix}\cos(\psi_i) & \sin(\psi_i)\end{bmatrix}^T$ with $\psi_i\in(-\pi, \pi]$ being the attitude \emph{yaw} angle\footnote{For our setup, the yaw angle and heading angle can be considered equal due to the absence of wind.} and $u_{\psi_i}$ is the control action that will make the aircraft to turn. In particular, for coordinated turns where the altitude of the vehicle is kept constant and the pitch angle is close to zero, the control action $u_{\psi_i}$ corresponds to the following bank angle $\phi_i$ to be tracked by the autopilot of the vehicle
\begin{equation}
	\phi_i = \operatorname{arctan}\frac{u_{\psi_i}}{g},
\end{equation}
where $g$ is the gravity acceleration.

\subsection{Trajectory tracking}
One of the key points of the proposed formation control algorithm in this paper is to make sure that the aircraft is tracking $\mathcal{C}_r$. There exist many guidance algorithms in the literature \cite{kendoul2012survey,sujit2014unmanned}. We have chosen the algorithm proposed in \cite{YuriCS} that has been successfully tested in real flights \cite{de2016guidance} for two reasons. Firstly, the local exponential converge to the desired path is guaranteed. This property will help us later to support the convergence of the formation control algorithm under the argument of \emph{slow-fast} dynamical systems in cascade. Secondly, the algorithm can be straightforwardly extended to other $C^2$ curves that are homeomorphic to $\mathcal{C}_r$, such as ellipses or the (possibly concave curve) Cassini ovals.

The trajectory tracking algorithm employs the level sets $e(p)\dfb\varphi(p)$ for the notion of \emph{error distance} between the aircraft and $\mathcal{C}_r$. Note that for circular paths, the error $e$ has a clear relation with the Euclidean distance, but for more general trajectories, such as ellipses, this is not always true. The vehicle has to follow the vector field defined by
\begin{equation}
	\dot p_d(p) \dfb \tau(p) - k_e e(p)n(p),
	\label{eq: gvf}
\end{equation}
where $k_e\in\R^+$ is a gain that defines how \emph{aggressive} the vector field is, in order to converge to traveling on $\mathcal{C}_r$. Let us define $\hat x$ as the unit vector constructed from the nonzero vector $x$.
\begin{theorem}
	\cite{YuriCS,de2016guidance} Consider the system (\ref{eq: pdyn}), then the control action
\begin{align}
	u_\psi &= -\left(E\hat{\dot p}_d\hat{\dot p}_d^TE\left((E-k_ee)H(\varphi)\dot p - k_en^T\dot pn\right)\right)^TE\frac{\dot p_d}{||\dot p_d||^2} \nonumber \\
	&+ k_d\hat{\dot p}^TE\hat{\dot p}_d,
	\label{eq: ui}
\end{align}
	where $H(\cdot)$ is the Hessian operator and $k_d\in\R^+$, makes the aircraft to converge (locally) exponentially fast to travel over $\mathcal{C}_r$, i.e., for $|e(0)|\leq c^*$ we have that $|e(t)| \leq a\operatorname{exp}(-bt)$ with $t\to\infty$ for some constants $a,b,c^*\in\R^+$.
	\label{th: tr}
\end{theorem}

The first term in (\ref{eq: ui}) makes the aircraft to stay on the guidance vector field (\ref{eq: gvf}) while the second term makes the vehicle to converge to the guidance vector field in case that the vehicle is not aligned with it.

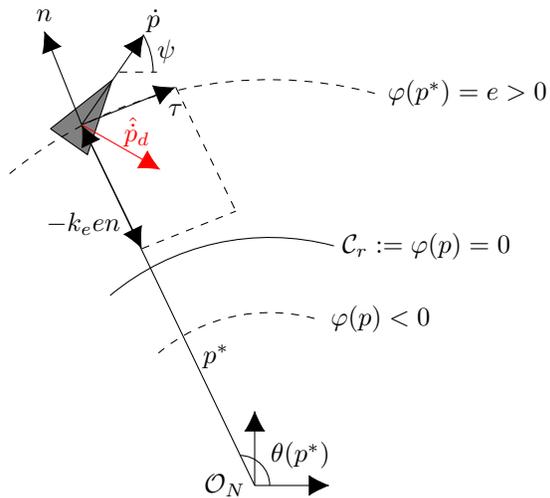
\begin{figure}
	\begin{tikzpicture}
	\draw[fill=gray,cm={cos(-55),-sin(-55),sin(-55),cos(-55),(1.15,0)}](0,1.3)--(0,0.7)--(1,1)--(0,1.3);
	\draw[dashed] (3,-4) ++(75:5.4) arc (75:130:5.4) node[at start, xshift=35]{$\varphi(p^*) = e > 0$};
		\draw[dashed] (3,-4) ++(75:2.3) arc (75:130:2.3) node[at start, xshift=25]{$\varphi(p) < 0$};
		\draw (3,-4) ++(75:3.3) arc (75:130:3.3) node[at start, xshift=35]{$\mathcal{C}_r:=\varphi(p) = 0$};;
		\draw (3,-4) arc (0:90:0.4) node[xshift=23]{$\theta(p^*)$};;
	\draw[-{Latex[length=8, width=8]}] (2.8, -4) -- (3.8, -4);
		\draw[-{Latex[length=8, width=8]}] (2.8, -4) node[left]{$\mathcal{O}_N$} -- (2.8, -3);
	\draw[-{Latex[length=8, width=8]}] (2.8, -4) -- (0.5,0.8) node[pos=0.3, above, xshift=5]{$p^*$};
	\draw[-{Latex[length=8, width=8]}] (0.5,0.8) -- (1.75, 1.3) node[pos=1, below,yshift=-3]{$\tau$};
	\draw[-{Latex[length=8, width=8]}] (0.5,0.8) -- (0, 2.05) node[pos=1, above]{$n$};
	\draw[-{Latex[length=8, width=8]}] (0.5,0.8) -- (1.3, -0.85) node[pos=0.8, left]{$-k_een$};
		\draw[-{Latex[length=8, width=8]}, color=red] (0.5,0.8) -- (3.05-1.5, 1.3-1.1) node[pos=0.7, above]{$\hat{\dot p}_d$};
	\draw[-{Latex[length=8, width=8]}] (0.5,0.8) -- (1.33,2) node[pos=1.15]{$\dot p$};
	\draw[dashed] (1, 1.5) -- (1.5, 1.5);
	\draw[dashed] (1.3, -0.85) -- (3.05-0.5, 1.3-0.85-0.8);
	\draw[dashed] (1.75, 1.3) -- (3.05-0.5, 1.3-0.85-0.8);
	\draw (1.45,1.5) arc (0:30:1) node[pos=-0, xshift=5, right, above]{$\psi$};
\end{tikzpicture}
	\caption{The direction to be followed by the UAV at the point $p^*$ for converging to $\mathcal{C}_r$ is given by $\hat{\dot p}_d$. The tangent and normal vectors $\tau$ and $n$ are calculated from $\nabla\varphi(p^*)$. The error \emph{distance} $e$ is calculated as $\varphi(p^*)$. All the circles $\varphi_c$ can be parametrized by an angle $\theta$ with respect to the horizontal axis of a frame of coordinates at the center of $\mathcal{C}_r$.}
	\label{fig: ilus}
\end{figure}

\section{Circular formation control}
\label{sec: cf}
\subsection{Problem definition}
Given the neighbors' relationship described by the graph $\mathcal{G}$, the stacked vector of inter-vehicle angles can be calculated as
\begin{equation}
z = B^T\theta,
\end{equation}
where $\theta\in\mathbb{T}^{|\mathcal{V}|}=\mathbb{S}^1\times\dots\times\mathbb{S}^1$ (the $n$-torus) is the stacked vector of parameters for each vehicle as in (\ref{eq: thetai}), and $z\in\mathbb{T}^{|\mathcal{E}|}$, therefore no necessarily all the inter-vehicle angles are calculated. Note that $z_k\in\mathbb{S}^1$ can be calculated from the relative measurement $p_i - p_j$ by just following trigonometric arguments in Figure \ref{fig: ilus}. Consider a collection of desired $z_k^*\in\mathbb{S}^1, k\in\{1,\dots,|\mathcal{E}|\}$ on the circle. We define the formation error $e_\theta\in\mathbb{T}^{|\mathcal{E}|}$ as the stacked vector of signals
\begin{equation}
	e_{\theta_k}(t) = z_k(t) - z_k^*,
	\label{eq: etheta}
\end{equation}
where $e_{\theta_k}\in\mathbb{S}^1$. The objective of the proposed algorithm for the team of fixed-wing UAVs in the next subsection is to achieve simultaneously $e_\theta(t) \to 0$ and $p_i(t)\to\mathcal{C}_r, \forall i\in\mathcal{N}$ as $t\to\infty$.

\subsection{Controller design and stability}
Consider that the unit speed aircraft $i$ is tracking correctly $\mathcal{C}_r$, therefore its angular velocity around the center of $\mathcal{C}_r$ is
\begin{equation}
	\dot\theta_i = \frac{1}{r}.
	\label{eq: 1r}
\end{equation}
The idea is to control the inter-vehicle angles in $z$ by changing in the vehicles the desired trajectory to be tracked, i.e., instead of (\ref{eq: cr}), the vehicle $i$ has to track
\begin{equation}
	^i\mathcal{C}_{(r,{^ic})} \dfb \{p \,:\, \varphi(p) =  {^ic}\},
\label{eq: cru}
\end{equation}
where ${^ic}\in\R$ is the formation control signal to be designed and the superindex $i$ denotes for the vehicle $i\in\mathcal{V}$. Note that the smaller the ${^ic}$ (possibly negative), the bigger the radius of $^i\mathcal{C}_{(r,{^ic})}$, and thus the smaller the angular velocity $\dot\theta_i$. For the sake of simplicity in the following analysis we define
\begin{equation}
	^ic \dfb {^iu_r^2}+2r\,{^iu_r},
	\label{eq: iciu}
\end{equation}
where ${^iu_r}\in\R$ is a control action with a more straightforward physical meaning than $^ic$, i.e., we set the radius (around the desired $r$) of the circumference $^i\mathcal{C}_{(r,{^ic})}$ (or simply $^i\mathcal{C}$) because
\begin{equation}
	x^2 + y^2 - r^2 = {^iu_r^2}+2r\,{^iu_r} \iff x^2 + y^2 - (r+\,{^iu_r})^2 = 0.
\end{equation}
\begin{remark}
Note that for a generic closed trajectory, if a vehicle tracks it on a negative level set, then it travels less distance than one that tracks the same trajectory on a positive level set after one loop.
\end{remark}

We propose the following control action for achieving the desired circular formation
\begin{equation}
	^iu_r = k_r B_ie,
	\label{eq: ur}
\end{equation}
where $B_i$ stands for the $i$'th row of the incidence matrix $B$ as in (\ref{eq: B}), and $k_r\in\R^+$. Since $e_{\theta_k} \in (-\pi,\pi]$, we impose to $k_r$ the following condition
\begin{equation}
	r - \pi k_r\,\max\limits_{i\in\mathcal{V}}(|\mathcal{N}_i|) > 0,
\end{equation}
i.e., we avoid the possibility of setting a negative radius\footnote{Note that while a level set can be negative, the radius of a circumference cannot.} in $^i\mathcal{C}$. Note that the control action (\ref{eq: ur}) is based on the popular consensus algorithm in formation control \cite{oh2015survey}.

Before presenting the main result, we need the following technical lemma that will define the neighbors' relationships.
\begin{lemma}
	If $\mathcal{G}$ does not contain any cycles, then the matrix $A\dfb -B^TB$ is Hurwitz.
	\label{lem: bb}
\end{lemma}
\begin{proof}
	If $\mathcal{G}$ does not contain any cycles, it has been show in \cite{guattery2000graph} that $Bx \neq 0$ for any nonzero vector $x\in\R^{|\mathcal{E}|}$. Note that $||Bx||^2 = x^TB^TBx>0$, implying that $B^TB$ is positive definite. Hence, $A$ is Hurwitz if $\mathcal{G}$ does not contain any cycles.
\end{proof}

We also make use of the following assumption.
\begin{assumption}
	A vehicle $i$ is always tracking and traveling over $^i\mathcal{C}$ as in (\ref{eq: cru}).
	\label{as: tr}
\end{assumption}

The Assumption \ref{as: tr} considers that if there is a change in the radius of $^i\mathcal{C}$, then the vehicle instantaneously \emph{jumps} to the required level set. As we will show, the circular formation controller (\ref{eq: ur}) guarantees the exponential stability of the origin of the error signal $e_\theta$ under the Assumption \ref{as: tr}. Since the trajectory error tracking $e$ in Theorem \ref{th: tr} is also locally exponentially stable, one may consider a \emph{slow-fast} dynamics in cascade by tuning appropriately the gains $k_e$ and $k_r$ in (\ref{eq: gvf}) and (\ref{eq: ur}) respectively \cite[Chapter 4]{sepulchre2012constructive}. Informally, the controller (\ref{eq: ui}) provides a \emph{fast} transient process of the vehicle to $^i\mathcal{C}$ if the aircraft is sufficiently close to it, while the whole formation \emph{slowly} follows the formation controller (\ref{eq: ur}). In fact, if $k_r$ is sufficiently small, the set of possible trajectories $^i\mathcal{C}$ will be very close to $\mathcal{C}_r$, therefore making reasonable the Assumption \ref{as: tr}. Furthermore, since for circles the convergence of the algorithm in Theorem \ref{th: tr} is almost globally stable (with the exception of starting at the center of $\mathcal{C}_r$), even if the vehicles start far away from the trajectories to be tracked, they will approach a situation where Assumption \ref{as: tr} can eventually be considered.

\begin{theorem}
	Consider a team of $n$ unit speeds aircraft modeled as in (\ref{eq: pdyn}), and the graph $\mathcal{G}$ defining the neighbors' relationships does not contain any cycles. All the vehicles are tracking (\ref{eq: cru}) by employing (\ref{eq: ui}). If Assumption \ref{as: tr} holds and the level sets $^ic$ in (\ref{eq: cru}) are controlled by (\ref{eq: ur}) via (\ref{eq: iciu}), then the origin of the error $e_\theta$ as in (\ref{eq: etheta}) is locally exponentially stable for the desired $z^*_k,k\in\{1,\dots,|\mathcal{E}|\}$.
	\label{th: fc}
\end{theorem}
\begin{proof}
	The proof is based on checking the stability of the linearization of the error dynamics $e_\theta$ around the origin. First note that $\dot e_\theta = \dot z = B^T\dot\theta$. According to Assumption \ref{as: tr}, we also have that for each edge $\mathcal{E}_k = (i,j)$ the agents $i$ and $j$ are tracking a circle of radius $r+k_rB_ie_\theta$ and $r+k_rB_je_\theta$ respectively, so from (\ref{eq: 1r}) it holds that
\begin{align}
	\dot e_{\theta_k} &= \dot z_k = \frac{1}{r+k_rB_ie_\theta} - \frac{1}{r+k_rB_je_\theta} \nonumber \\
	& = \frac{k_r(B_j - B_i)e_\theta}{(r+k_rB_ie_\theta)(r+k_rB_je_\theta)} \nonumber \\
	& = \frac{k_rA_ke_\theta}{(r+k_rB_ie_\theta)(r+k_rB_je_\theta)}, \quad k\in\{1,\dots,|\mathcal{E}|\},
	\label{eq: ekdyn}
\end{align}
	where $A_k$ is the $k$'th row of the matrix $A$ as in Lemma \ref{lem: bb}. We linearize (\ref{eq: ekdyn}) around $e_\theta = 0$, therefore the dynamics of small variations $\epsilon_{\theta_k}$ of the error are given by
	\begin{align}
		\dot\epsilon_{\theta_k} = \frac{\partial \dot e_{\theta_k}}{\partial e_{\theta}}\Big|_{e_\theta = 0} \epsilon_{\theta} = \frac{k_r}{r^2}A_k\epsilon_{\theta},
	\end{align}
which leads to the following compact form
\begin{equation}
\dot\epsilon_\theta =  \frac{k_r}{r^2}A\epsilon_\theta,
\end{equation}
and because $A$ is Hurwitz according to Lemma \ref{lem: bb} since $\mathcal{G}$ has not any cycles, we can conclude that the equilibrium $e_\theta = 0$ is locally exponentially stable.
\end{proof}
\begin{remark}
	Note that since the convergence to the trajectories $^i\mathcal{C}$ is asymptotic, one can guarantee that all the vehicles will be confined in a disc $\mathcal{D}$ of radius $(r+\pi k_r\,\max\limits_{i\in\mathcal{V}}(|\mathcal{N}_i|)$, which corresponds to the worst case radius to be tracked, for all time $t$, even if the control $u_r$ is not updated, e.g., the vehicles are not exchanging or sensing their positions.
\end{remark}

It is interesting to highlight that if $\mathcal{G}$ does not contain any cycles, then in such a disc $\mathcal{D}$ the only equilibrium point for the system is at $e_\theta = 0$, which has been proven stable. Since the vehicles are eventually confined in $\mathcal{D}$ according to Theorem \ref{th: tr}, it seems reasonable to conjecture that an estimation of the region of attraction for the exponentially stable $e_\theta = 0$ is indeed $\mathcal{D}$. Furthermore, for a proof of the convergence of the overall system without Assumption \ref{as: tr}, one can use the stability theory of cascade systems \cite[Chapter 4]{sepulchre2012constructive}, while the exponential stability of the partial system (\ref{eq: ekdyn}) could guarantee the (locally) asymptotic stability of the overall system. We will present a rigorous proof in the extended journal version.

\section{Implementation and flight performance}
\label{sec: exp}
\subsection{Experimental platform}
The validity of Theorem \ref{th: fc} has been tested with the three fixed-wing UAVs shown in Figure \ref{fig: 3uav}. The aircraft have about $600$ grams of weight, $1.2$ m of wingspan, and they are actuated by two elevons and one motor. The electronics include a battery that allows about $45$ minutes of autonomy at the nominal flight, which corresponds to an airspeed of $13$ m/s. The chosen board for running the Paparazzi autopilot stack is the Apogee \cite{papa}, which includes the usual sensors of three axis gyros, accelerometers, magnetometers. Each fixed-wing UAV has on board an U-Blox GPS with a nominal accuracy of $3$ meters in the horizontal plane. The airplanes exchange their positions according to $\mathcal{G}$, so they can compute the corresponding inter-vehicle angles $z_k$. The source code can be checked online at the Paparazzi repository \cite{papa}.

\begin{figure}
\centering
\includegraphics[width=1\columnwidth]{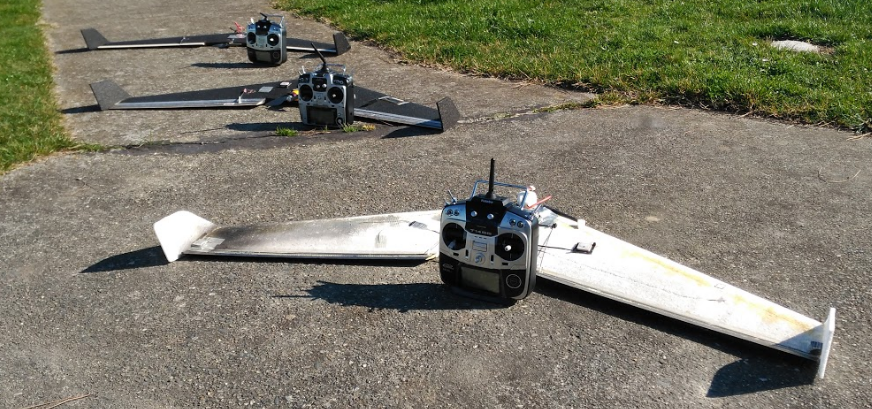}
	\caption{The three fixed-wing UAVs employed for the circular formation at the aero club of Eole at Muret (Toulouse).}
\label{fig: 3uav}
\end{figure}
\subsection{Circular formation flight experiment}
The formation flight\footnote{The video of the experiment with HD quality can be watched at https://www.youtube.com/c/HectorGarciadeMarina .} has been taken place at the aero model club of Eole at Muret, close to the city of Toulose in France. The flights were performed on the 17th of February, 2017 between the 10:00 and the 12:00 hours local time. The wind coming from the south was about $2$ m/s according to MeteoFrance, therefore we can consider that ground-speed and airspeed are approximately equal. The airplanes $1,2$ and $3$ are tagged with the colors blue, pink and red respectively at the ground station captions in Figure \ref{fig: gs}. The chosen incidence matrix for the communication between vehicles is
\begin{equation}
	B = \begin{bmatrix}
		1 & 0 \\
		-1 & 1 \\
		0 & -1
	\end{bmatrix},
\end{equation}
which clearly does not define any cycles. The desired formation is defined by $z_1^*=z_2^*=0$, i.e., all the aircraft achieve consensus for their corresponding $\theta_i$. Potential collisions are avoided by making the airplanes to fly at different altitudes, which are $30, 32$ and $35$ meters above the ground for planes $1, 2$ and $3$ respectively. The target circle $\mathcal{C}_r$ is set at the waypoint CIRCLE in Figure \ref{fig: gs} with a desired radius $r$ equal to $80$ meters. The gains $k_d, k_e$ and $k_r$ in Theorems \ref{th: tr} and \ref{th: fc} have been set to $1, 1$ and $8$ respectively. The airplanes exchange their positions with a frequency of $2$ Hz, although lost communications are expected. Note that each airplane has a different understanding of $e_{\theta_{\{1,2\}}}$, i.e., each airplane calculates on board the error signal and it might be different among neighbors due to lost communications or GPS delays. Interesting work studying the effects of this fact in formation control can be found at \cite{mou2016undirected,iros2017}.

Before the formation control algorithm begins the three aircraft are orbiting at different standby points. The experiment starts at time $131$ seconds in Figure \ref{fig: gs}. At this moment the algorithm commands the airplane $2$ (red) to follow a circumference with a much smaller radius than airplanes $1$ (blue) and $3$ (pink) in order \emph{to catch them up}. In fact, airplanes $1$ and $3$ are tracking a circumference with a bigger radius than $\mathcal{C}_r$ in order \emph{to wait} for airplane $2$. In about $15$ seconds, the errors $e_{\theta_{\{1,2\}}}$ have been reduced half. Some lost communications have been experienced between times $150$ and $170$ seconds. However, the algorithm seems robust against such an issue and the formation achieves consensus within a band of $\pm10$ degrees of error, and continues stable until the end of the experiment after seven laps.

\begin{figure*}
\centering
\begin{subfigure}{.49\columnwidth}
  \centering
  \includegraphics[width=\linewidth]{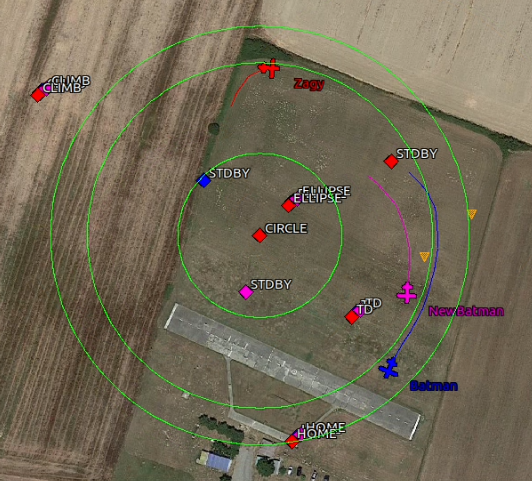}
  \caption{t = $131$secs}
\end{subfigure}
\begin{subfigure}{.49\columnwidth}
  \centering
  \includegraphics[width=\linewidth]{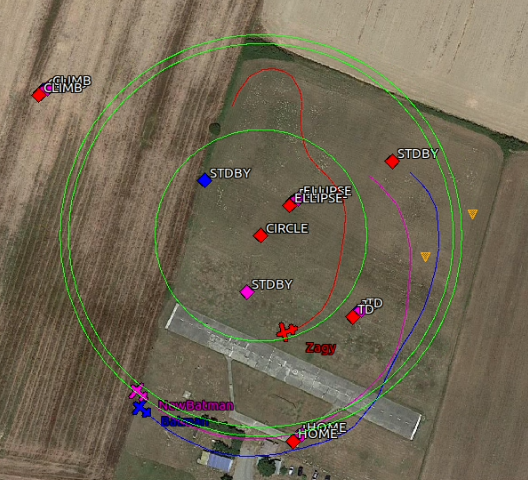}
   \caption{t = $145$secs}
\end{subfigure}
\begin{subfigure}{.49\columnwidth}
  \centering
  \includegraphics[width=\linewidth]{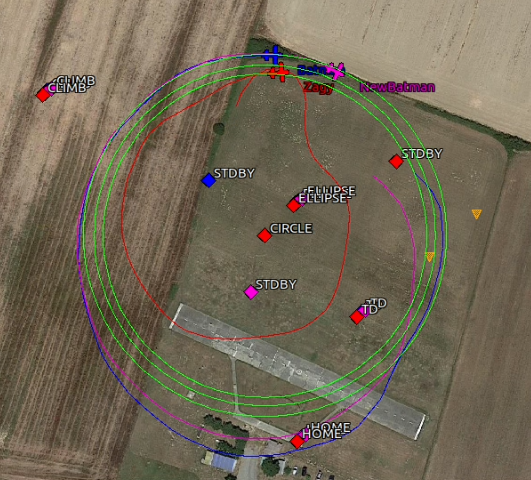}
   \caption{t = $162$secs}
\end{subfigure}
\begin{subfigure}{.49\columnwidth}
  \centering
  \includegraphics[width=\linewidth]{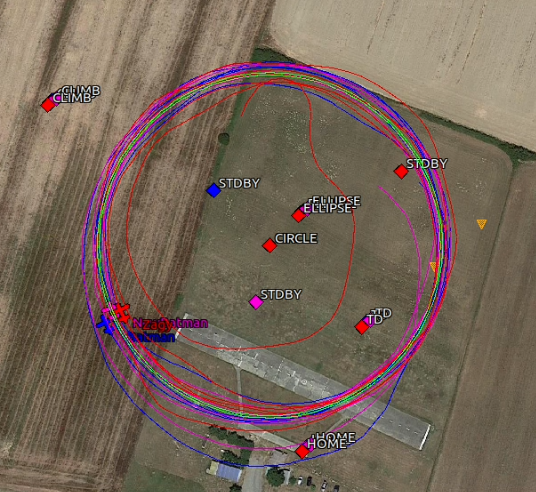}
	\caption{t = $410$secs} 
\end{subfigure}%
\caption{Screenshots from the Paparazzi ground control station showing the evolution of the circular formation.}
\label{fig: gs}
\end{figure*}

\begin{figure}
\centering
\includegraphics[width=1\columnwidth]{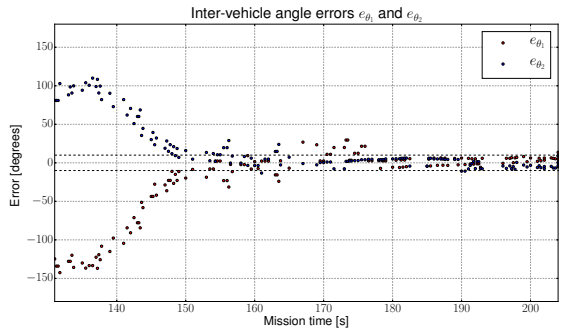}
	\caption{Inter-vehicle error angle evolution. The zone $\pm10$ degrees is between the dashed lines. Each dot represents the calculated error by one of the airplanes once it receives the position of its neighbor. The density of points is different due to lost communications. However, the algorithm seems robust enough against such an issue as it drives both errors close to zero within a band $\pm10$ degrees.} 
\label{fig: error}
\end{figure}

\section{Conclusions}
	This paper has presented an algorithm for achieving circular formations with fixed-wing UAVs traveling with constant speeds. The strategy consists in controlling the angular velocities around the center of the desired circle. For that, we design a control action that is applied to the level set to be tracked around the desired circle. These level sets are tracked by a guidance vector field algorithm that is locally exponentially stable. Since the presented algorithm for circular formations is also exponentially stable if the vehicles are perfectly tracking the level sets, we employ the argument of \emph{slow-fast} systems in cascade in order to show the compatibility of both algorithms, which has been demonstrated in practice with three aircraft. The algorithm can be potentially extended to other closed non-circle trajectories. A more rigorous analysis will be presented in an extension of this work.



\bibliographystyle{IEEEtran}
\bibliography{hector_ref.bib}

\begin{thebibliography}{10}
\providecommand{\url}[1]{#1}
\csname url@rmstyle\endcsname
\providecommand{\newblock}{\relax}
\providecommand{\bibinfo}[2]{#2}
\providecommand\BIBentrySTDinterwordspacing{\spaceskip=0pt\relax}
\providecommand\BIBentryALTinterwordstretchfactor{4}
\providecommand\BIBentryALTinterwordspacing{\spaceskip=\fontdimen2\font plus
\BIBentryALTinterwordstretchfactor\fontdimen3\font minus
  \fontdimen4\font\relax}
\providecommand\BIBforeignlanguage[2]{{%
\expandafter\ifx\csname l@#1\endcsname\relax
\typeout{** WARNING: IEEEtran.bst: No hyphenation pattern has been}%
\typeout{** loaded for the language `#1'. Using the pattern for}%
\typeout{** the default language instead.}%
\else
\language=\csname l@#1\endcsname
\fi
#2}}

\bibitem{beard2006decentralized}
R.~W. Beard, T.~W. McLain, D.~B. Nelson, D.~Kingston, and D.~Johanson,
  ``Decentralized cooperative aerial surveillance using fixed-wing miniature
  {UAVs},'' \emph{Proceedings of the IEEE}, vol.~94, no.~7, pp. 1306--1324,
  2006.

\bibitem{goodrich2008supporting}
M.~A. Goodrich, B.~S. Morse, D.~Gerhardt, J.~L. Cooper, M.~Quigley, J.~A.
  Adams, and C.~Humphrey, ``Supporting wilderness search and rescue using a
  camera-equipped mini {UAV},'' \emph{Journal of Field Robotics}, vol.~25, no.
  1-2, pp. 89--110, 2008.

\bibitem{tokekar2016sensor}
P.~Tokekar, J.~Vander~Hook, D.~Mulla, and V.~Isler, ``Sensor planning for a
  symbiotic {UAV} and {UGV} system for precision agriculture,'' \emph{IEEE
  Transactions on Robotics}, vol.~32, no.~6, pp. 1498--1511, 2016.

\bibitem{olfati2007consensus}
R.~Olfati-Saber, J.~A. Fax, and R.~M. Murray, ``Consensus and cooperation in
  networked multi-agent systems,'' \emph{Proceedings of the IEEE}, vol.~95,
  no.~1, pp. 215--233, 2007.

\bibitem{ryan2004overview}
A.~Ryan, M.~Zennaro, A.~Howell, R.~Sengupta, and J.~K. Hedrick, ``An overview
  of emerging results in cooperative {UAV} control,'' in \emph{Decision and
  Control, 2004. CDC. 43rd IEEE Conference on}, vol.~1.\hskip 1em plus 0.5em
  minus 0.4em\relax IEEE, 2004, pp. 602--607.

\bibitem{bertuccelli2004robust}
L.~Bertuccelli, M.~Alighanbari, and J.~How, ``Robust planning for coupled
  cooperative {UAV} missions,'' in \emph{Decision and Control, 2004. CDC. 43rd
  IEEE Conference on}, vol.~3.\hskip 1em plus 0.5em minus 0.4em\relax IEEE,
  2004, pp. 2917--2922.

\bibitem{wang2007cooperative}
X.~Wang, V.~Yadav, and S.~Balakrishnan, ``Cooperative {UAV} formation flying
  with obstacle/collision avoidance,'' \emph{IEEE Transactions on control
  systems technology}, vol.~15, no.~4, pp. 672--679, 2007.

\bibitem{stipanovic2004decentralized}
D.~M. Stipanovi{\'c}, G.~Inalhan, R.~Teo, and C.~J. Tomlin, ``Decentralized
  overlapping control of a formation of unmanned aerial vehicles,''
  \emph{Automatica}, vol.~40, no.~8, pp. 1285--1296, 2004.

\bibitem{oh2015survey}
K.-K. Oh, M.-C. Park, and H.-S. Ahn, ``A survey of multi-agent formation
  control,'' \emph{Automatica}, vol.~53, pp. 424--440, 2015.

\bibitem{anderson2008uav}
B.~D.~O. Anderson, B.~Fidan, C.~Yu, and D.~Walle, ``{UAV} formation control:
  Theory and application,'' in \emph{Recent advances in learning and
  control}.\hskip 1em plus 0.5em minus 0.4em\relax Springer, 2008, pp. 15--33.

\bibitem{montenbruck2016fekete}
J.~M. Montenbruck, D.~Zelazo, and F.~Allg{\"o}wer, ``Fekete points, formation
  control, and the balancing problem,'' \emph{arXiv preprint arXiv:1606.08203},
  2016.

\bibitem{lee2016distributed}
B.-H. Lee and H.-S. Ahn, ``Distributed formation control via global orientation
  estimation,'' \emph{Automatica}, vol.~73, pp. 125--129, 2016.

\bibitem{lin2005necessary}
Z.~Lin, B.~Francis, and M.~Maggiore, ``Necessary and sufficient graphical
  conditions for formation control of unicycles,'' \emph{IEEE Transactions on
  automatic control}, vol.~50, no.~1, pp. 121--127, 2005.

\bibitem{dimarogonas2007rendezvous}
D.~V. Dimarogonas and K.~J. Kyriakopoulos, ``On the rendezvous problem for
  multiple nonholonomic agents,'' \emph{IEEE Transactions on automatic
  control}, vol.~52, no.~5, pp. 916--922, 2007.

\bibitem{marshall2006pursuit}
J.~A. Marshall, M.~E. Broucke, and B.~A. Francis, ``Pursuit formations of
  unicycles,'' \emph{Automatica}, vol.~42, no.~1, pp. 3--12, 2006.

\bibitem{el2013distributed}
M.~I. El-Hawwary and M.~Maggiore, ``Distributed circular formation
  stabilization for dynamic unicycles,'' \emph{IEEE Transactions on Automatic
  Control}, vol.~58, no.~1, pp. 149--162, 2013.

\bibitem{leonard2010coordinated}
N.~E. Leonard, D.~A. Paley, R.~E. Davis, D.~M. Fratantoni, F.~Lekien, and
  F.~Zhang, ``Coordinated control of an underwater glider fleet in an adaptive
  ocean sampling field experiment in monterey bay,'' \emph{Journal of Field
  Robotics}, vol.~27, no.~6, pp. 718--740, 2010.

\bibitem{anderson2010fundamentals}
J.~D. Anderson~Jr, \emph{Fundamentals of aerodynamics}.\hskip 1em plus 0.5em
  minus 0.4em\relax Tata McGraw-Hill Education, 2010.

\bibitem{justh2004equilibria}
E.~W. Justh and P.~Krishnaprasad, ``Equilibria and steering laws for planar
  formations,'' \emph{Systems \& control letters}, vol.~52, no.~1, pp. 25--38,
  2004.

\bibitem{sepulchre2007stabilization}
R.~Sepulchre, D.~A. Paley, and N.~E. Leonard, ``Stabilization of planar
  collective motion: All-to-all communication,'' \emph{IEEE Transactions on
  Automatic Control}, vol.~52, no.~5, pp. 811--824, 2007.

\bibitem{sepulchre2008stabilization}
------, ``Stabilization of planar collective motion with limited
  communication,'' \emph{IEEE Transactions on Automatic Control}, vol.~53,
  no.~3, pp. 706--719, 2008.

\bibitem{seyboth2014collective}
G.~S. Seyboth, J.~Wu, J.~Qin, C.~Yu, and F.~Allg{\"o}wer, ``Collective circular
  motion of unicycle type vehicles with non-identical constant velocities,''
  \emph{IEEE Transactions on Control of Network Systems}, vol.~1, no.~2, pp.
  167--176, 2014.

\bibitem{brinon2014cooperative}
L.~Bri{\~n}on-Arranz, A.~Seuret, and C.~C. De~Wit, ``Cooperative control design
  for time-varying formations of multi-agent systems,'' \emph{IEEE Transactions
  on Automatic Control}, vol.~59, no.~8, pp. 2283--2288, 2014.

\bibitem{sun2015collective}
Z.~Sun, G.~S. Seyboth, and B.~D.~O. Anderson, ``Collective control of multiple
  unicycle agents with non-identical constant speeds: Tracking control and
  performance limitation,'' in \emph{Control Applications (CCA), 2015 IEEE
  Conference on}.\hskip 1em plus 0.5em minus 0.4em\relax IEEE, 2015, pp.
  1361--1366.

\bibitem{wang2013forming}
C.~Wang, G.~Xie, and M.~Cao, ``Forming circle formations of anonymous mobile
  agents with order preservation,'' \emph{IEEE Transactions on Automatic
  Control}, vol.~58, no.~12, pp. 3248--3254, 2013.

\bibitem{bullo2009distributed}
F.~Bullo, J.~Cort{\'e}s, and S.~Martinez, \emph{Distributed control of robotic
  networks: a mathematical approach to motion coordination algorithms}.\hskip
  1em plus 0.5em minus 0.4em\relax Princeton University Press, 2009.

\bibitem{wang2014controlling}
C.~Wang, G.~Xie, and M.~Cao, ``Controlling anonymous mobile agents with
  unidirectional locomotion to form formations on a circle,''
  \emph{Automatica}, vol.~50, no.~4, pp. 1100--1108, 2014.

\bibitem{hattenberger2014using}
G.~Hattenberger, M.~Bronz, and M.~Gorraz, ``Using the paparazzi {UAV} system
  for scientific research,'' in \emph{IMAV 2014, International Micro Air
  Vehicle Conference and Competition 2014}, 2014, pp. pp--247.

\bibitem{kendoul2012survey}
F.~Kendoul, ``Survey of advances in guidance, navigation, and control of
  unmanned rotorcraft systems,'' \emph{Journal of Field Robotics}, vol.~29,
  no.~2, pp. 315--378, 2012.

\bibitem{sujit2014unmanned}
P.~Sujit, S.~Saripalli, and J.~B. Sousa, ``Unmanned aerial vehicle path
  following: A survey and analysis of algorithms for fixed-wing unmanned aerial
  vehicless,'' \emph{IEEE Control Systems}, vol.~34, no.~1, pp. 42--59, 2014.

\bibitem{YuriCS}
\BIBentryALTinterwordspacing
Y.~A. Kapitanyuk, A.~V. Proskurnikov, and M.~Cao, ``A guiding vector field
  algorithm for path following control of nonholonomic mobile robots,'' 2016,
  iEEE Transactions on Control System Technology, accepted. [Online].
  Available: \url{https://arxiv.org/abs/1610.04391}
\BIBentrySTDinterwordspacing

\bibitem{de2016guidance}
H.~G. de~Marina, Y.~A. Kapitanyuk, M.~Bronz, G.~Hattenberger, and M.~Cao,
  ``Guidance algorithm for smooth trajectory tracking of a fixed wing uav
  flying in wind flows,'' in \emph{Accepted in the 2017 IEEE International
  Conference on Robotics and Automation (ICRA)}, 2017.

\bibitem{guattery2000graph}
S.~Guattery and G.~L. Miller, ``Graph embeddings and laplacian eigenvalues,''
  \emph{SIAM Journal on Matrix Analysis and Applications}, vol.~21, no.~3, pp.
  703--723, 2000.

\bibitem{sepulchre2012constructive}
R.~Sepulchre, M.~Jankovic, and P.~V. Kokotovic, \emph{Constructive nonlinear
  control}.\hskip 1em plus 0.5em minus 0.4em\relax Springer Science \& Business
  Media, 2012.

\bibitem{papa}
\BIBentryALTinterwordspacing
Paparazzi. {UAV} open-source project. [Online]. Available:
  \url{http://wiki.paparazziuav.org/}
\BIBentrySTDinterwordspacing

\bibitem{mou2016undirected}
S.~Mou, M.-A. Belabbas, A.~S. Morse, Z.~Sun, and B.~D.~O. Anderson,
  ``Undirected rigid formations are problematic,'' \emph{IEEE Transactions on
  Automatic Control}, vol.~61, no.~10, pp. 2821--2836, 2016.

\bibitem{iros2017}
H.~G.~de Marina, J.~Siemonsma, B.~Jayawardhana, and M.~Cao, ``Design and
  implementation of formation control algorithms for fully distributed
  multi-robot systems,'' in \emph{Intelligent Robots and Systems (IROS), 2017
  IEEE/RSJ International Conference on, Submitted}, 2017.

\end{thebibliography}

\end{document}